\PassOptionsToPackage{table,dvipsnames}{xcolor}
\documentclass[]{bytedance_seed}

\usepackage{amsmath}
\usepackage{amssymb}
\usepackage{amsfonts}
\usepackage{mathtools}
\usepackage{amsthm}
\usepackage{array}
\usepackage{tabularx}
\usepackage{xspace}
\usepackage{pifont}
\usepackage{lscape}
\usepackage{enumitem}
\usepackage{float}
\usepackage{wrapfig}
\usepackage{algorithm}
\usepackage{algorithmic}

\newcolumntype{C}[1]{>{\centering\arraybackslash}p{#1}}

\definecolor{mygreen}{HTML}{009901}

\newcommand{\sysname}{\texttt{DLA}\xspace}

\theoremstyle{plain}
\newtheorem{theorem}{Theorem}[section]

\newtheorem{corollary}[theorem]{Corollary}
\theoremstyle{definition}

\theoremstyle{remark}

\title{Dynamic Linear Attention}

\author[1,*]{Xin Wang}
\author[2,*]{Hui Shen}
\author[2]{Boyuan Zheng}
\author[2]{Xueshen Liu}
\author[2]{Minkyoung Cho}
\author[1]{Zhongwei Wan}
\author[2]{Zesen Zhao}
\author[2]{Zhuoqing Mao}
\author[3]{Shen Yan}
\author[1]{Mi Zhang}

\affiliation[1]{The Ohio State University}
\affiliation[2]{University of Michigan}
\affiliation[3]{ByteDance Seed}
\contribution{\textsuperscript{*}Equal contribution.}

\abstract{The scalability of Large Language Models (LLMs) to long contexts is fundamentally constrained by the quadratic complexity of standard attention, motivating the adoption of linear attention mechanisms with sub-quadratic cost. To improve representation capacity under long contexts, recent approaches organize memory in a multi-state manner. However, existing multi-state linear attention methods rely on fixed state merging policies that cannot adapt to dynamically varying token importance, irreversibly obscuring critical tokens and causing severe error accumulation over long sequences.
To address this limitation, we propose \sysname, a dynamic memory modeling framework for multi-state linear attention. \sysname introduces (i) \textbf{\emph{Information-Aware Dynamic State Merging}}, which adaptively determines state boundaries based on token-level information variation, preserving high-resolution representations around semantic transitions while aggressively summarizing stable regions, and (ii) \textbf{\emph{Capacity-Bounded Memory Modeling}}, which maintains a fixed-size, chronologically ordered state cache by selectively merging adjacent low-information states to control memory growth with minimal information loss.
We pre-train \sysname on two different linear attention models and evaluate on 16 datasets across three categories. Experimental results demonstrate the superiority of \sysname over state-of-the-art.
}

\correspondence{Xin Wang at \email{wang.15980@osu.edu}, Mi Zhang at \email{mizhang.1@osu.edu}}

\begin{document}

\maketitle

\section{Introduction}
\label{sec:introduction}

Large Language Models (LLMs) have demonstrated remarkable capabilities across a wide range of natural language understanding and generation tasks. However, scaling LLMs to long-context settings remains a fundamental challenge due to the quadratic computational and memory complexity of standard self-attention~\citep{wan2023efficient, wang2024iot, DBLP:conf/iclr/WanWZXTZWLXW025}. This limitation has motivated extensive research on efficient attention mechanisms that enable long-sequence modeling without retraining from scratch. Among these approaches, linear attention~\citep{DBLP:conf/nips/YangWZSK24, DBLP:conf/iclr/YangKH25} has emerged as a promising direction, as it approximates full attention with sub-quadratic complexity and offers favorable scalability to long contexts.

To further improve the representation capacity of linear attention under long sequences, recent works organize historical context in a multi-state manner, where long token histories are partitioned into chunks and summarized into compact memory states. Representative methods such as Log-Linear Attention~\citep{DBLP:journals/corr/abs-2506-04761} demonstrate improved efficiency and practicality for long-context inference. By operating on summarized states rather than individual tokens, these approaches significantly reduce memory footprint and computation cost.

Despite their success, existing multi-state linear attention methods still suffer from notable performance degradation as context length increases. This limitation stems from a fundamental mismatch between fixed memory construction policies and the non-uniform, dynamically evolving information structure of long sequences. In particular, current methods typically rely on fixed block sizes or rule-based merging schedules, implicitly assuming uniform information density across the sequence. Such designs fail to adapt to dynamically emerging semantic transitions, forcing critical tokens to be prematurely absorbed into coarse summaries. Moreover, merge decisions made under fixed policies are irreversible: once heterogeneous tokens are compressed into a single state, their individual contributions cannot be recovered, leading to error accumulation.

These observations suggest that effective long-context linear attention requires memory modeling mechanisms that are both information-aware and capacity-controlled. On one hand, state construction should adapt to local representation variation, allocating higher resolution to semantically volatile regions while aggressively summarizing stable spans. On the other hand, the total number of memory states must be explicitly bounded to ensure predictable computation and memory cost during inference.

In this work, we propose Dynamic Linear Attention (\sysname), a new framework for multi-state linear attention that addresses these challenges. \sysname differs from prior approaches in two key aspects.
First, \sysname introduces Information-Aware Dynamic State Merging, which determines state boundaries on the fly based on token-level information variation. Instead of relying on fixed merging policies, \sysname evaluates the representation change of each incoming token relative to the current memory state, merging low-variation tokens while initiating new states at semantic transition points.
Second, \sysname incorporates Capacity-Bounded Memory Modeling, which maintains a fixed-size, chronologically ordered state cache. When the cache reaches its capacity, \sysname selectively merges adjacent low-information states, preserving temporal order while minimizing information loss.

We pre-train \sysname on two linear-attention backbones, Mamba-2-780M and Gated DeltaNet-1.3B, following the design in~\citep{DBLP:journals/corr/abs-2506-04761}. We evaluate \sysname on 16 datasets spanning three aspects: eight commonsense reasoning benchmarks, six in-context retrieval datasets, and two long-context modeling datasets.
We highlight three main findings. (1) \sysname consistently outperforms the state-of-the-art multi-state method, Log-Linear Attention, across all tasks. (2) When applied to Mamba-2, the \sysname variant even achieves performance comparable to full-attention Transformers with similar parameter budgets. (3) \sysname achieves superior efficiency, delivering higher throughput and lower runtime memory consumption than Log-Linear Attention. 

\section{Preliminary}
\label{sec:preliminary}

We consider a sequence modeling task with input length $T$ and hidden dimension $d$.
Let $\mathbf{Q}, \mathbf{K}, \mathbf{V} \in \mathbb{R}^{T \times d}$ denote the query, key, and value matrices.
Standard self-attention computes the output $\mathbf{O} \in \mathbb{R}^{T \times d}$ as
\[
\mathbf{O} = \mathrm{softmax}(\mathbf{Q}\mathbf{K}^\top \odot \mathbf{M})\mathbf{V},
\]
where $\mathbf{M}$ is the causal mask.
While effective, this operation incurs quadratic time and memory complexity in $T$,
motivating the development of sub-quadratic attention mechanisms.
In this section, we review linear attention and its multi-state variants that form the foundation of our approach.

\subsection{Linear Attention}

Linear attention mitigates the quadratic cost of Transformers by removing the softmax normalization,
enabling the reordering of computation via associativity.
A causal linear attention layer can be written in a parallel form as
\begin{equation}
\mathbf{O} = (\mathbf{Q}\mathbf{K}^\top \odot \mathbf{M})\mathbf{V},
\qquad \mathbf{M}_{ij} = \mathbb{I}(i \ge j).
\label{eq:linear_attn_parallel}
\end{equation}
This formulation admits an equivalent recurrent implementation.
Let $\mathbf{q}_t, \mathbf{k}_t, \mathbf{v}_t \in \mathbb{R}^d$ denote the query, key, and value vectors at time step $t$.
Linear attention maintains a single state matrix $\mathbf{S}_t \in \mathbb{R}^{d \times d}$
that summarizes all past tokens:
\begin{align}
\mathbf{S}_t &= \mathbf{S}_{t-1} + \mathbf{v}_t \mathbf{k}_t^\top, \\
\mathbf{o}_t &= \mathbf{S}_t \mathbf{q}_t.
\end{align}
This recurrent form enables linear-time inference with constant memory,
but compresses the entire history into a single state,
which can limit representation capacity under long contexts. We use $\phi(\cdot)$ to denote the feature map used in linear attention.
Unless otherwise specified, $\phi:\mathbb{R}^d \rightarrow \mathbb{R}^d$ is implemented as an identity mapping or a learnable linear projection, following prior work.

\subsection{Linear Attention with the Delta Rule}

To improve state tracking and introduce controlled forgetting,
DeltaNet~\citep{DBLP:conf/nips/YangWZSK24} extends linear attention with a delta-style update rule:
\begin{equation}
\mathbf{S}_t
=
\mathbf{S}_{t-1}(\mathbf{I} - \beta_t \mathbf{k}_t \mathbf{k}_t^\top)
+ \mathbf{v}_t \mathbf{k}_t^\top,
\label{eq:delta_rule_rnn}
\end{equation}
where $\beta_t$ is a data-dependent step size.
While this formulation improves adaptivity over a pure accumulator,
it still relies on a single global state and therefore cannot selectively preserve fine-grained information over long sequences.

\subsection{Multi-State Linear Attention}

To increase modeling capacity while retaining sub-quadratic complexity,
recent work organizes linear attention in a multi-state manner by partitioning the historical context
into segments and summarizing each segment into a separate state~\citep{DBLP:journals/corr/abs-2506-04761,DBLP:journals/corr/abs-2507-04416}.
Among them, log-linear attention~\citep{DBLP:journals/corr/abs-2506-04761}
replaces the single recurrent state with a logarithmic number of multi-scale states
constructed via a Fenwick-tree decomposition of the causal prefix. Concretely, at time step $t$, the prefix $[0,t]$ is decomposed into at most
$L=\lceil \log_2(t+1)\rceil+1$ disjoint buckets $\{B_t^{(\ell)}\}_{\ell=0}^{L-1}$,
with finer resolution near the current position and coarser resolution for distant history. The corresponding linear attention states and the final aggregated output are computed as:
\begin{equation}
\mathbf{S}_t^{(\ell)} = \sum_{s \in B_t^{(\ell)}} \mathbf{v}_s \mathbf{k}_s^\top \in \mathbb{R}^{d \times d}, \mathbf{o}_t = \sum_{\ell=0}^{L-1} \lambda_t^{(\ell)} \, \mathbf{S}_t^{(\ell)} \mathbf{q}_t
\end{equation}
This design achieves $O(T\log T)$ training complexity and $O(\log T)$ time and memory per decoding step. However, the granularity of its memory states is determined by a fixed hierarchical schedule, independent of token-level representation variation.
As a result, semantically salient tokens may be prematurely absorbed into coarse summaries, and disturbances introduced at critical positions can propagate through the fixed multi-scale states. This limitation motivates the need for information-aware and adaptive memory construction, which we address in the next section.



\section{Dynamic Linear Attention (DLA)}
\label{sec:methodology}

\begin{figure*}[t]
\centering
\includegraphics[width=1.0\textwidth]{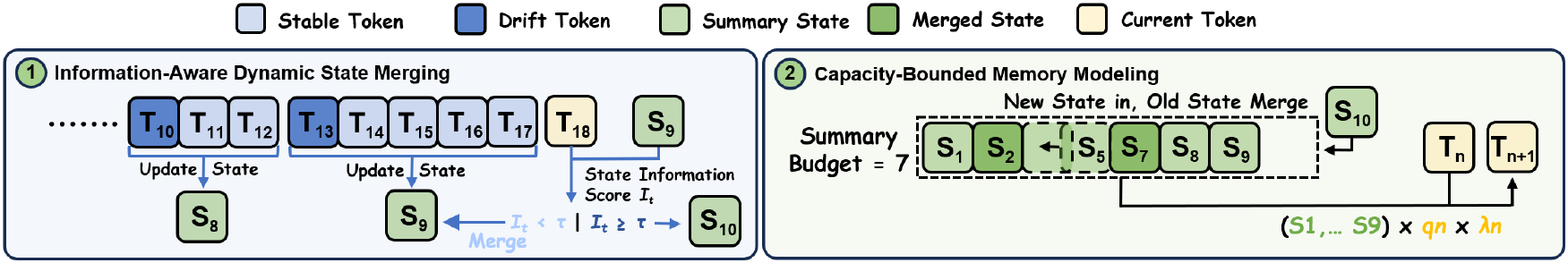}
\caption{Overview of \sysname.}
\label{fig:framework}
\end{figure*}

\cref{fig:framework} provides an overview of \sysname.
\sysname is an information-aware linear attention framework that dynamically constructs a compact set of memory states for efficient long-context modeling. Unlike prior approaches that rely on fixed temporal schedules or predefined block boundaries, \sysname adaptively determines state granularity based on token-level information variation.
Specifically, tokens are processed sequentially. For each new token, \sysname computes a lightweight \emph{State Information Score} measuring its representation change relative to the most recent memory state. Tokens with low information variation are merged into the current state, while tokens exhibiting significant drift initiate a new state. This enables fine-grained modeling around semantic transitions while aggressively summarizing stable token spans.
To bound memory and computation, \sysname maintains a capacity-limited state cache. When the cache reaches its maximum size, two adjacent states with the lowest information density are merged, preserving temporal order while minimizing information loss. The resulting memory consists of a fixed-size, chronologically ordered set of summary states.
At each decoding step, \sysname produces the output by attending over the maintained memory states using a linear attention formulation, where each state contributes with a query-dependent weight. Together, information-aware state construction and capacity-bounded memory modeling enable \sysname to achieve adaptive resolution, stable inference cost, and efficient long-context representation.

\begin{algorithm}[t]
\caption{Information-Aware Dynamic State Merging}
\label{alg:information_aware_dynamic_state_merging}
\begin{algorithmic}[1]
    \STATE \textbf{Input:} Token States $\{s_t\}_{t=1}^T$
    \STATE \textbf{Output:} memory states $\mathcal{M}=\{S_i\}$

    \STATE $\mathcal{M} \gets [\ ]$
    \FOR{$t = 1$ \textbf{to} $T$}
        \IF{$\mathcal{M}$ is empty}
            \STATE $\mathcal{M} \gets \{s_t\}$; \textbf{continue}
        \ENDIF
        \STATE $S \gets$ last state in $\mathcal{M}$
        \STATE $I_t  \gets \frac{\|s_t - S\|_F}{\|S\|_F + \epsilon}$
        \STATE replace last state in $\mathcal{M}$ with $\mathrm{Merge}(S,s_t)$
        \IF{$I_t \ge \tau$}
            \STATE append $s_t$ to $\mathcal{M}$; 
        \ENDIF
        
    \ENDFOR
    \STATE \textbf{return} $\mathcal{M}$
\end{algorithmic}
\end{algorithm}

\subsection{Information-Aware Dynamic State Merging}
\label{subsec:information_aware_dynamic_state_merging}
\textbf{Motivation:} Existing multi-block linear attention methods typically rely on fixed schedules (e.g., block and merge every $K$ tokens)~\citep{DBLP:journals/corr/abs-2506-04761} or hard, rule-based boundaries~\citep{DBLP:journals/corr/abs-2507-04416} to determine the block of historical tokens that should be merged into summary states. While such designs improve memory and compute efficiency, they are largely agnostic to the semantic evolution of the sequence. In practice, information density is highly non-uniform: critical semantic transitions may occur abruptly, whereas long stretches of tokens can be locally redundant. As a result, fixed or hard block policies often suffer from \textbf{two key limitations}. First, they cannot adapt to dynamically emerging semantic changes, forcing important transitions to be prematurely absorbed into coarse summaries simply because a pre-defined boundary is reached. Second, merge decisions made without regard to local semantic continuity are inherently irreversible: once tokens are merged under a fixed policy, their individual contributions cannot be recovered, even if subsequent context reveals their importance. These misalignments between merge decisions and the true semantic structure lead to sub-optimal generation and finally degrades the representation quality. 

In the following, we provide a theoretical proof on why the fixed merging policy is sub-optimal.
\begin{theorem}[State deviation]
\label{thm:summarization_deviation}
Let $\{u_t\}_{t=1}^T \subset \mathbb{R}^d$ denote per-token additive contributions to a linear attention state.
Consider a blocking policy $\pi$ of token list $\{1,\dots,T\}$ into $m$ disjoint contiguous blocks
$\{\mathcal{C}_i\}_{i=1}^m$. For each block $\mathcal{C}_i$, let $\bar{u}_i \in \mathbb{R}^d$ be a representative summary vector. Therefore, for any query vector $q \in \mathbb{R}^d$, the exact output $y(q)$ and the summarized output $\tilde y_{\pi}(q)$ for a query vector are: 
\begin{equation}
\label{eq:summarized_y}
y(q) \triangleq \sum_{t=1}^T \langle q, u_t\rangle, \quad \tilde y_\pi(q) \triangleq \sum_{i=1}^m \sum_{t\in\mathcal{C}_i} \langle q, \bar{u}_i\rangle
\end{equation}
The deviation induced by summarization $\operatorname{Err}(\pi ; q) \triangleq\left|y(q)-\tilde{y}_\pi(q)\right|$ admits the bound:
\begin{equation}
\label{eq:bound_main}
\big|y(q)-\tilde y_\pi(q)\big|
\;\le\;
\|q\|_2 \cdot \sum_{i=1}^m \sqrt{|\mathcal{C}_i|}
\;\sqrt{\sum_{t\in\mathcal{C}_i}\|u_t-\bar{u}_i\|_2^2}
\end{equation}
\end{theorem}
\begin{proof}
The deviation between the exact and summarized outputs $y(q)-\tilde y(q)$ can be further rewritten as:
\begin{align}
\sum_{i=1}^m \sum_{t\in\mathcal{C}_i} \langle q, u_t-\bar{u}_i\rangle
=
\left\langle q,\; \sum_{i=1}^m \sum_{t\in\mathcal{C}_i} (u_t-\bar{u}_i)\right\rangle
\end{align}

By Applying Cauchy--Schwarz~\citep{johnston2025generalizing} Inequality, we have:
\begin{equation}
\label{eq:cs1}
\operatorname{Err}(\pi ; q) = \big|y(q)-\tilde y_\pi(q)\big|
\le
\|q\|_2 \cdot
\left\|\sum_{i=1}^m \sum_{t\in\mathcal{C}_i} (u_t-\bar{u}_i)\right\|_2
\end{equation}
We then use Triangle Inequality~\citep{10.1145/1644893.1644914} over chunks to further get:
\begin{equation}
\label{eq:tri_chunks}
\left\|\sum_{i=1}^m \sum_{t\in\mathcal{C}_i} (u_t-\bar{u}_i)\right\|_2
\le
\sum_{i=1}^m
\left\|\sum_{t\in\mathcal{C}_i} (u_t-\bar{u}_i)\right\|_2
\end{equation}
Similarly, for each block $\mathcal{C}_i$, we also have:
\begin{align}
\left\|\sum_{t\in\mathcal{C}_i} (u_t-\bar{u}_i)\right\|_2
&\le
\sum_{t\in\mathcal{C}_i} \|u_t-\bar{u}_i\|_2 \notag \\
&\le
\sqrt{|\mathcal{C}_i|}\cdot
\sqrt{\sum_{t\in\mathcal{C}_i}\|u_t-\bar{u}_i\|_2^2}.
\label{eq:cs2}
\end{align}
By combining~\eqref{eq:cs1}, \eqref{eq:tri_chunks}, and~\eqref{eq:cs2}, we finally obtain the upper-bound $B(\pi ; q)$ of the deviation $\operatorname{Err}(\pi ; q)$:
\begin{equation}
\label{eq:bound}
B(\pi ; q) \triangleq\|q\|_2 \sum_{i=1}^m \sqrt{\left|\mathcal{C}_i\right|} \sqrt{\sum_{t \in \mathcal{C}_i}\left\|u_t-\bar{u}_i\right\|_2^2}
\end{equation}
This upper bound shows that the deviation induced by block-wise summarization is controlled by the within-block heterogeneity. As a result, content-agnostic fixed blocking policies, which do not adapt to representation variation, can incur a larger bound on non-stationary sequences, especially when tokens with large representation variance are mixed into the same block.
\end{proof}

\begin{corollary}[Fixed blocking is sub-optimal on non-stationary sequences]
\label{cor:fixed_suboptimal}
There exists a class of non-stationary token sequences for which any fixed blocking
policy $\pi_{\mathrm{fix}}$ yields a strictly larger deviation bound $B(\pi_{\mathrm{fix}};q)$
than an adaptive blocking policy $\pi_{\mathrm{dyn}}$ that aligns block boundaries with
semantic change points.
\end{corollary}

\begin{proof}[Proof sketch]
Consider a non-stationary sequence consisting of two contiguous segments
$\mathcal{A}$ and $\mathcal{B}$ with distinct means $\mu_A \neq \mu_B$.
For simplicity, assume $u_t=\mu_A$ for $t\in\mathcal{A}$ and $u_t=\mu_B$ for $t\in\mathcal{B}$
(a special case of $u_t\sim\mathcal{D}_A,\mathcal{D}_B$).

Let $\pi_{\mathrm{fix}}$ be any fixed blocking policy that yields at least one block
$\mathcal{C}$ overlapping both segments, and denote
$n_A = |\mathcal{C}\cap\mathcal{A}|$, $n_B = |\mathcal{C}\cap\mathcal{B}|$.
For this block, the choice $\bar u$ that minimizes $\sum_{t\in\mathcal{C}}\|u_t-\bar u\|_2^2$
is the block mean $\bar u = \frac{n_A\mu_A+n_B\mu_B}{n_A+n_B}$, and the minimum within-block
heterogeneity satisfies
\[
\sum_{t\in\mathcal{C}}\|u_t-\bar u\|_2^2
=
\frac{n_A n_B}{n_A+n_B}\,\|\mu_A-\mu_B\|_2^2
> 0 .
\]
In contrast, an adaptive policy $\pi_{\mathrm{dyn}}$ that places a boundary at the change point
produces blocks contained in $\mathcal{A}$ or $\mathcal{B}$ only, for which the optimal heterogeneity term is $0$ under the same construction. Since the deviation bound $B(\pi;q)$ is a sum of nonnegative per-block terms
$\|q\|_2\sqrt{|\mathcal{C}_i|}\sqrt{\sum_{t\in\mathcal{C}_i}\|u_t-\bar u_i\|_2^2}$,
the overlapping block $\mathcal{C}$ alone contributes a strictly positive amount to
$B(\pi_{\mathrm{fix}};q)$ for any $q\neq 0$, while $B(\pi_{\mathrm{dyn}};q)$ does not incur
this cross-segment term. Hence, there exists such a sequence for which
$B(\pi_{\mathrm{fix}};q) > B(\pi_{\mathrm{dyn}};q)$, proving the claim.
\end{proof}

\begin{figure*}[t]
\centering
\includegraphics[width=1\textwidth]{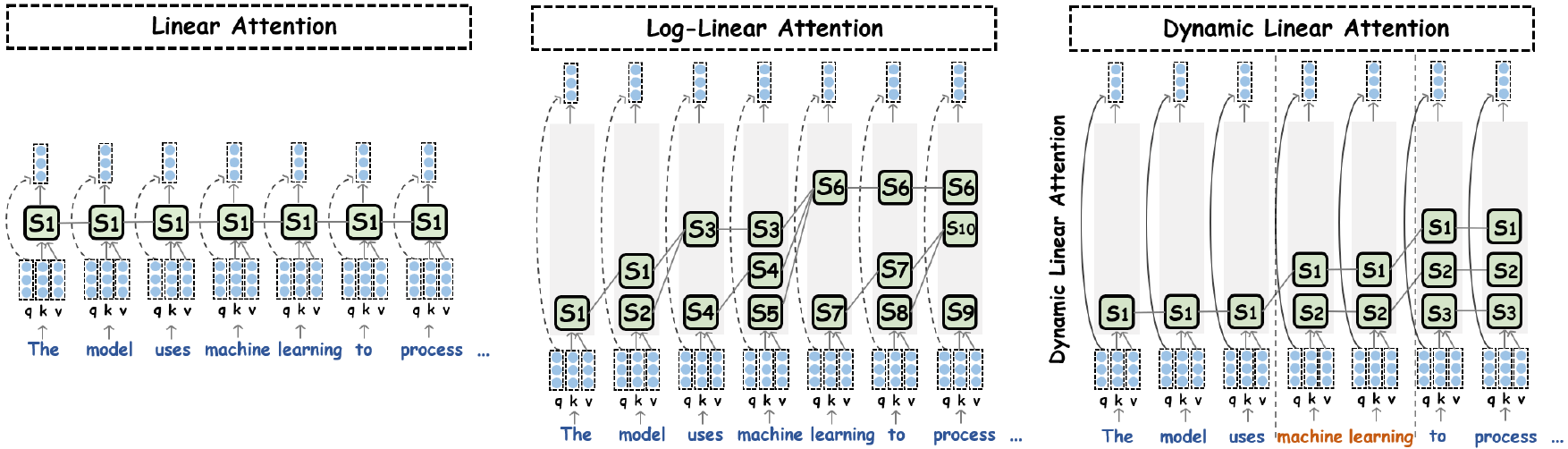}
\caption{Standard linear attention (left) vs. log-linear attention (mid) vs. dynamic linear attention (right). The input consists of query, key, and value vectors.}
\label{fig:comparison}
\end{figure*}
\textbf{Key Design:} The pseudocode of Information-Aware Dynamic State Merging of \sysname is provided in~\cref{alg:information_aware_dynamic_state_merging}. We also plot the difference between Vanilla Linear Attention, Log-Linear Attention, and our Dynamic Linear Attention (\sysname) in~\cref{fig:comparison}. Specifically, to dynamically determine whether a newly generated token $t$ should be merged into an existing memory state or initiate a new one, we first introduce a new metric named \textit{State Information Score} ($I_t$) to measure the amount of novel information carried by the current token relative to the most recent memory state. Concretely, let $s_{t} \triangleq \phi\left(k_t\right) v_t^{\top}$ denote the state of new token $t$, and let $S_{t-1}$ denote the previous memory state, which summarizes multiple past tokens. We quantify the information variation between $S_t$ and $S_{t-1}$ as follows:
\begin{align}
\label{eq:importance_score}
I_t = \frac{\| S_t - S_{t-1} \|_F}{\| S_{t-1} \|_F + \epsilon}
\end{align}
In practice, we apply RMSNorm to both $S_t$ and $S_{t-1}$ prior to score computation to further stabilize the scale across layers and timesteps. 
During inference, we measure the following boundary indicator
\begin{equation}
b_t \triangleq \mathbf{1}\left[I_t \geq \tau\right]
\end{equation}
Let $S_{t-1}^{\text {cur }}$ denote the most recent memory state in the cache. The state update rule at inference is then defined as
\begin{equation}
S_t^{\text {cur }}= \begin{cases}S_{t-1}^{\text {cur }} + s_t, & b_t=0, \\ S_t, & b_t=1,\end{cases}
\end{equation}
where $b_t=1$ indicates that the current token initiates a new memory state, while $b_t=0$ continues to accumulate information into the existing state. When $b_t=1$, the newly created state $S_t$ is appended to the memory cache, preserving the chronological order of states.

We apply soft gating to decide the boundary in a differentiable manner during pre-training and then switch to a hard segmentation strategy during inference to ensure that inference produces a discrete set of memory states aligned with semantic boundaries, while retaining the same information-aware criterion learned during training.
\paragraph{Discussion.}
Theorem~\ref{thm:summarization_deviation} shows the summarization deviation is dominated by the within-block heterogeneity term in~\cref{eq:bound}. Fixed blocking policies are content-agnostic and therefore may mix tokens from distinct semantic regimes into the same block, which yields a strictly larger deviation bound on non-stationary sequences (~\cref{cor:fixed_suboptimal}). In contrast, \sysname monitors token-level representation drift and only merges a new token when the induced increase of heterogeneity is small, using the State Information Score $I_t$ in~\cref{eq:importance_score}. Therefore, \sysname can be viewed as a greedy online strategy that approximately minimizes the dominant term in the deviation bound, while fixed policies ignore it, making it less competitive than \sysname.

\begin{algorithm}[t]
\caption{Capacity-Bounded Memory Modeling}
\label{alg:capacity_bounded_memory_modeling}
\begin{algorithmic}[1]
    \STATE \textbf{Input:} Incoming states $\{S_t\}$, Information Scores $\{\bar I_t\}$, Token Counts $\{n_t\}$, Capacity $K$, Queries $\{q_t\}$
    \STATE \textbf{Output:} attention outputs $\{o_t\}$

    \STATE $\mathcal{M} \gets [\ ]$ \COMMENT{state cache}
    \FOR{each $(S_i,\bar I_i,n_i)$ in time order}
        \IF{$|\mathcal{M}| = K$}
            \STATE $(i^\star,i^\star\!+\!1) \gets \arg\min_i \frac{\bar I_i+\bar I_{i+1}}{n_i+n_{i+1}}$
            \STATE $S_{i^\star} \gets S_{i^\star} + S_{i^\star+1}$
            \STATE $n_{i^\star} \gets n_{i^\star} + n_{i^\star+1}$
            \STATE $\bar{I}_{i^\star} \gets \bar{I}_{i^\star} + \bar{I}_{i^\star+1}$
            \STATE remove $S_{i^\star+1}$ from $\mathcal{M}$
        \ENDIF
        \STATE append $(S_i,\bar I_i,n_i)$ to $\mathcal{M}$
    \ENDFOR
    \STATE $o_t = \sum_{i} \phi(q_t)\, S_i$
    \STATE \textbf{return} $\{o_t\}$
\end{algorithmic}
\end{algorithm}

\subsection{Capacity-Bounded Memory Modeling}
\label{subsec:capacity_bounded_memory_modeling}
\textbf{Motivation:} 
While the previous design enables flexible and information-aware memory state construction, maintaining an unbounded number of states is impractical for efficient inference, especially in long-context or high-throughput serving scenarios, as dynamic memory growth leads to irregular memory layouts, variable attention costs, and reduced batching efficiency. To address these challenges, it is essential to explicitly limit the number of memory states while preserving the most informative summaries. 

\textbf{Key Design:} The pseudocode of Capacity-Bounded Memory Modeling of \sysname is provided in~\cref{alg:capacity_bounded_memory_modeling}. Specifically, \sysname maintains a state cache $\mathcal{M} = \{(S_i, n_i, \bar{I}_i)\}_{i=1}^{m}$ with $m \le K$, where $S_i \in \mathbb{R}^d$ denotes the $i$-th memory state in chronological order, $n_i$ is the number of tokens summarized by $S_i$, and $\bar{I}_i$ is an aggregated information score of all tokens in this state. We maintain $\bar{I}_i$ as the sum of per-token information scores within each state, such that $\bar{I}_i / n_i$ measures information density. Newly generated tokens are first converted to per-token representations $S_t$, and a tentative state is produced following~\cref{subsec:information_aware_dynamic_state_merging}. The resulting state is appended to the cache, preserving temporal order. When the cache is not full ($m < K$), we simply insert the new state. When the cache reaches capacity ($m = K$), we trigger a compression step that merges two \emph{adjacent} states to free one slot. Restricting merges to adjacent states preserves the temporal order and avoids distorting positional semantics. Concretely, among all consecutive pairs $(i, i\!+\!1)$, we select the pair with the lowest information density:
\begin{align}
(i^\star, i^\star\!+\!1) = \arg\min_{i \in \{1,\dots,K-1\}}
\frac{\bar{I}_i + \bar{I}_{i+1}}{n_i + n_{i+1}}
\end{align}
We then merge them using a summarization operator as in~\cref{subsec:information_aware_dynamic_state_merging},
\begin{align}
S_{i^\star} &\leftarrow S_{i^\star} + S_{i^\star+1}, \notag \\
n_{i^\star} &\leftarrow n_{i^\star} + n_{i^\star+1}, \quad
\bar{I}_{i^\star} \leftarrow \bar{I}_{i^\star} + \bar{I}_{i^\star+1}
\end{align}
and shift the remaining states accordingly to keep $m=K-1$ before inserting the incoming state.

Given the capacity-bounded cache $\mathcal{M}$, we compute the output at time step $t$ by attending over the stored memory states. Let $q_t$ denote the query vector of the current token. The final output is then computed as 
\begin{equation}
o_t
= \sum_{i=1}^{m} \lambda_{t,i} \,
q_t^\top
\left(
\sum_{s \in \mathcal{C}_i} v_s k_s^\top
\right)
= \sum_{i=1}^{m} \lambda_{t,i} \,
q_t^\top S_i ,
\label{eq:capacity_bounded_output}
\end{equation}
where $\lambda_{t,i}$ is the weight learned during pre-training with the same shape as the memory capacity. Following~\citep{DBLP:journals/corr/abs-2506-04761}, $\lambda_{t,i}$ is produced by a learned linear layer over the query representation and reused at inference. In this way, \sysname provides a unified way to read from a temporally ordered, capacity-bounded memory, enabling stable inference cost while retaining the ability to emphasize informative states during the inference.

\section{Experiments}
\label{sec:experiments}

\begin{table*}[t]
\centering
\caption{Performance comparison of \sysname and baseline methods on zero-shot commonsense reasoning tasks on Mamba-2 (780M) and Gated DeltaNet (1.3B). Commonsense reasoning datasets (LAMBADA, PIQA, HellaSwag, WinoGrande, ARC-e, ARC-c, OpenBookQA, and CommonsenseQA) are measured by accuracy (\textcolor{ForestGreen}{$\uparrow$}).  The best performance is marked in bold. The relative performance gain compared to the best-performing baseline is marked in green inside bracket. }
\label{tab:lm_cs}
\resizebox{\linewidth}{!}{%
\begin{tabular}{r|cccccccc|c}
\toprule
\textbf{Model}
& \textbf{LMB.\textcolor{ForestGreen}{$\uparrow$}} 
& \textbf{PIQA\textcolor{ForestGreen}{$\uparrow$}} 
& \textbf{Hella.\textcolor{ForestGreen}{$\uparrow$}} 
& \textbf{Wino.\textcolor{ForestGreen}{$\uparrow$}} 
& \textbf{ARC-e\textcolor{ForestGreen}{$\uparrow$}} 
& \textbf{ARC-c\textcolor{ForestGreen}{$\uparrow$}}
& \textbf{OBQA\textcolor{ForestGreen}{$\uparrow$}}
& \textbf{CSQA\textcolor{ForestGreen}{$\uparrow$}}
& \textbf{Average\textcolor{ForestGreen}{$\uparrow$}} \\
\midrule
Transformer  & 21.8 & 63.1 & 30.3 & 50.9 & 44.2 & 17.7 & 16.8 & 18.0 & 32.9 \\
\midrule
Mamba-2  & 15.7 & 58.9 & 29.3 & 50.1 & 46.0 & 18.9 & 15.4 & 20.3 & 31.8 \\
 w/ Log-linear   & 13.2 & 59.7 & 27.8 & 49.5 & 42.3 & 20.1 & 16.0 & 19.1 & 31.0 \\
\rowcolor{gray!15}
 w/ \sysname   &
 \textbf{18.7 (\textcolor{ForestGreen}{$\uparrow19\%$})} &
 \textbf{63.7 (\textcolor{ForestGreen}{$\uparrow7\%$})} &
 \textbf{30.8 (\textcolor{ForestGreen}{$\uparrow5\%$})} &
 \textbf{51.5 (\textcolor{ForestGreen}{$\uparrow3\%$})} &
 \textbf{48.1 (\textcolor{ForestGreen}{$\uparrow5\%$})} &
 \textbf{22.1 (\textcolor{ForestGreen}{$\uparrow10\%$})} &
 \textbf{17.4 (\textcolor{ForestGreen}{$\uparrow9\%$})} &
 \textbf{21.1 (\textcolor{ForestGreen}{$\uparrow4\%$})} &
 \textbf{34.2 (\textcolor{ForestGreen}{$\uparrow8\%$})} \\
\midrule
Gated DeltaNet & 20.3 & 58.8 & 29.6 & 51.3 & 44.7 & 20.2 & 16.0 & 21.3 & 32.8 \\
 w/ Log-linear   & 19.0 & 60.4 & 28.4 & 51.9 & 44.3 & 20.5 & 15.4 & 21.0 & 32.6 \\
\rowcolor{gray!15}
 w/ \sysname   &
 \textbf{20.8 (\textcolor{ForestGreen}{$\uparrow22\%$})} &
 \textbf{62.7 (\textcolor{ForestGreen}{$\uparrow4\%$})} &
 \textbf{30.3 (\textcolor{ForestGreen}{$\uparrow2\%$})} &
 \textbf{52.4 (\textcolor{ForestGreen}{$\uparrow2\%$})} &
 \textbf{45.0 (\textcolor{ForestGreen}{$\uparrow1\%$})} &
 \textbf{23.0 (\textcolor{ForestGreen}{$\uparrow14\%$})} &
 \textbf{17.4 (\textcolor{ForestGreen}{$\uparrow9\%$})} &
 \textbf{22.9 (\textcolor{ForestGreen}{$\uparrow8\%$})} &
 \textbf{34.8 (\textcolor{ForestGreen}{$\uparrow6\%$})} \\
\bottomrule
\end{tabular}%
}
\end{table*}

\subsection{Experimental Setups}
\label{subsec:setup}

\textbf{Baselines.} We compare \sysname against two groups of models: (1) Vanilla linear attention models, including Mamba-2-780M~\citep{DBLP:conf/icml/DaoG24} and Gated DeltaNet-1.3B~\citep{DBLP:conf/iclr/YangKH25}. (2) Multi-state linear attention models, including Mamba-2 with Log-linear Attention, and Gated DeltaNet with Log-linear Attention~\citep{DBLP:journals/corr/abs-2506-04761}. Following the design in previous work~\citep{DBLP:journals/corr/abs-2506-04761}, we also compare the \sysname version of Mamba-2-780M with full attention Transformers with 24 layers and 778M parameters.

\textbf{Datasets.} To demonstrate the generalizability of \sysname, we evaluate the performance of \sysname on 16 datasets covering three categories,
%
%
including eight commonsense reasoning datasets (LAMBADA~\citep{DBLP:conf/acl/PapernoKLPBPBBF16}, PIQA~\citep{DBLP:conf/aaai/BiskZLGC20}, HellaSwag~\citep{DBLP:conf/acl/ZellersHBFC19}, WinoGrande~\citep{DBLP:journals/cacm/SakaguchiBBC21}, OpenBookQA~\citep{DBLP:conf/emnlp/MihaylovCKS18}, CommonsenseQA~\citep{DBLP:conf/naacl/TalmorHLB19}, ARC-e, and ARC-c~\citep{DBLP:journals/corr/abs-2102-03315}),
six in-context retrieval datasets (SWDE~\citep{lockard2019openceres}, SQuAD~\citep{rajpurkar2018know}, FDA~\citep{arora2023language}, TriviaQA~\citep{joshi2017triviaqa}, Drop~\citep{dua2019drop}, NQ~\citep{kwiatkowski2019natural}),
and two long-context datasets (RULER~\citep{DBLP:journals/corr/abs-2404-06654} and LongBench~\citep{DBLP:conf/acl/BaiLZL0HDLZHDTL24}). All of the evaluations are conducted using the LM-Evaluation-Harness framework~\citep{eval-harness}.

\textbf{Implementation Details.} To ensure a fair comparison, we followed the same configuration used in Log-Linear Attention~\citep{DBLP:journals/corr/abs-2506-04761} to train the full attention Transformer-778M, Mamba-2-780M, Gated DeltaNet-1.3B, and their variants in Log-Linear and \sysname forms. Specifically,
%
%
we perform academic-scale language modeling pretraining from scratch using 50B tokens on the Long-Data-Collections dataset, using a sequence length of 16K. 
We set the capacity of the state cache in \sysname to 30, which is the same as the maximum state number in Log-linear attention.
All of our experiments are conducted on 4 NVIDIA A100 GPUs.

\subsection{Overall Comparison}
\label{subsec:overall_comparison}
We evaluate the overall performance of \sysname from three main aspects: (1) performance on commonsense reasoning tasks, (2) performance on in-context retrieval tasks, and (3) performance on long-context modeling tasks.


\textbf{Performance on Commonsense Reasoning.} Following prior work~\citep{DBLP:conf/icml/DaoG24}, we evaluate all models on eight commonsense reasoning benchmarks. Results are summarized in~\cref{tab:lm_cs}.
We make two key observations.
First, \sysname consistently outperforms both the vanilla and log-linear variants of linear-attention–based models across all tasks. In particular, compared to the log-linear variant, \sysname achieves up to 52\% and 22\% relative accuracy improvement on Mamba-2-780M and Gated DeltaNet-1.3B, respectively.
Second, when applied to Mamba-2-780M, \sysname also consistently outperforms a full-attention Transformer with a comparable parameter size, demonstrating that \sysname can close and even surpass the accuracy gap between linear attention and full attention.

\begin{table}[t]
\centering
\caption{Performance on in-context retrieval benchmarks measured by accuracy (\textcolor{ForestGreen}{$\uparrow$}). The best performance is marked in bold. The relative performance gain compared to the best-performing baseline is marked in green inside bracket.}
\label{tab:in-context-retrieval}
\resizebox{0.8\textwidth}{!}{%
\begin{tabular}{r|cccccc|c}
\toprule
\textbf{Model}
& \textbf{SQuAD\textcolor{ForestGreen}{$\uparrow$}} 
& \textbf{TriviaQA\textcolor{ForestGreen}{$\uparrow$}} 
& \textbf{SWDE\textcolor{ForestGreen}{$\uparrow$}} 
& \textbf{FDA\textcolor{ForestGreen}{$\uparrow$}} 
& \textbf{NQ\textcolor{ForestGreen}{$\uparrow$}} 
& \textbf{Drop\textcolor{ForestGreen}{$\uparrow$}}
& \textbf{Avg.\textcolor{ForestGreen}{$\uparrow$}} \\
\midrule
Mamba-2   & 22.4 & 13.6 & 17.4 & 0.0 & 0.2 & 3.4 & 9.5 \\
 w/ Log-linear   & 7.1 & 9.1 & 15.2 & 0.0 & 0.0 & 6.6 & 6.3 \\
\rowcolor{gray!15}
  w/ \sysname   &
  \textbf{28.1 (\textcolor{ForestGreen}{$\uparrow25\%$})} &
  \textbf{20.2 (\textcolor{ForestGreen}{$\uparrow49\%$})} &
  \textbf{26.3 (\textcolor{ForestGreen}{$\uparrow51\%$})} &
  \textbf{0.0 (\textcolor{ForestGreen}{$\uparrow0\%$})} &
  \textbf{1.0 (\textcolor{ForestGreen}{$\uparrow400\%$})} &
  \textbf{6.6 (\textcolor{ForestGreen}{$\uparrow0\%$})} &
  \textbf{13.7 (\textcolor{ForestGreen}{$\uparrow44\%$})}\\
\midrule
Gated DeltaNet & 15.8 & 29.4 & 20.3 & 16.9 & 10.7 & 13.5 & 17.8 \\
 w/ Log-linear   & 21.7 & 25.3 & 28.8 & 17.7 & 10.1 & 18.3 & 20.3 \\
\rowcolor{gray!15}
 w/ \sysname   &
 \textbf{27.3 (\textcolor{ForestGreen}{$\uparrow26\%$})} &
 \textbf{33.8 (\textcolor{ForestGreen}{$\uparrow15\%$})} &
 \textbf{51.7 (\textcolor{ForestGreen}{$\uparrow80\%$})} &
 \textbf{22.2 (\textcolor{ForestGreen}{$\uparrow25\%$})} &
 \textbf{12.1 (\textcolor{ForestGreen}{$\uparrow13\%$})} &
 \textbf{25.8 (\textcolor{ForestGreen}{$\uparrow41\%$})} &
 \textbf{28.8 (\textcolor{ForestGreen}{$\uparrow42\%$})}\\
\bottomrule
\end{tabular}%
}
\end{table}
\textbf{Performance on In-Context Retrieval Tasks.} Then, we evaluate the models on six in-context retrieval tasks following prior work~\citep{arora2024simple}. As shown in~\cref{tab:in-context-retrieval}, \sysname consistently outperforms the baseline methods with at most 49$\%$  performance improvement.

\begin{wraptable}{r}{0.60\textwidth}
\vspace{-2mm}
\centering
\caption{Evaluation results of single-needle tasks (S-NIAH-1--3) and multi-needle tasks (MK-1, MQ, MV) on \textbf{RULER} (4K context).}
\label{tab:ruler}
\resizebox{0.6\textwidth}{!}{%
\begin{tabular}{r|cccccc}
\toprule
\textbf{Model} & \textbf{S-NIAH-1} & \textbf{S-NIAH-2} & \textbf{S-NIAH-3} & \textbf{MK-NIAH-1} & \textbf{MQ-NIAH} & \textbf{MV-NIAH} \\
\midrule
Transformer               & 91.4 & 71.4 & 73.3 & 62.1 & 54.5 & 36.9 \\
\midrule
Mamba-2                   & 80.5 & 31.0 & 4.7 & 14.8 & 19.3 & 19.4 \\
\quad w/ Log-Linear       & 88.7  & 39.4 & 10.6 & 41.6 & 25.2 & 28.3 \\
\rowcolor{gray!15}
\quad w/ \sysname         & 100.0 (\textcolor{ForestGreen}{$\uparrow13\%$}) & 69.2 (\textcolor{ForestGreen}{$\uparrow76\%$}) & 37.1 (\textcolor{ForestGreen}{$\uparrow250\%$}) & 60.3 (\textcolor{ForestGreen}{$\uparrow45\%$}) & 42.5 (\textcolor{ForestGreen}{$\uparrow67\%$}) & 37.6 (\textcolor{ForestGreen}{$\uparrow33\%$}) \\
\midrule
Gated DeltaNet            & 100.0 & 74.0 & 53.2 & 19.1 & 19.5 & 14.8 \\
\quad w/ Log-Linear       & 100.0 & 81.6 & 48.8 & 43.8 & 31.3 & 26.6 \\
\rowcolor{gray!15}
\quad w/ \sysname         & 100.0 (\textcolor{ForestGreen}{$\uparrow0\%$})& 98.6 (\textcolor{ForestGreen}{$\uparrow21\%$})& 65.9 (\textcolor{ForestGreen}{$\uparrow24\%$})& 49.3 (\textcolor{ForestGreen}{$\uparrow13\%$}) & 34.2 (\textcolor{ForestGreen}{$\uparrow9\%$}) & 30.5 (\textcolor{ForestGreen}{$\uparrow15\%$}) \\
\bottomrule
\end{tabular}%
}
\vspace{-2mm}
\end{wraptable}

\textbf{Performance on Long-Context Modeling Tasks.} We next evaluated the models on long-context tasks, including long-context retrieval on RULER with 4k, 8k, 16k  length and long-context understanding on LongBench. As shown in~\cref{tab:ruler} and~\cref{tab:long_bench}, we make two main observations.

First, \sysname substantially improves long-context retrieval performance on RULER across both single-needle and multi-needle settings. Compared to the log-linear variant, \sysname achieves consistent and often large gains on Mamba-2 and Gated DeltaNet, with particularly pronounced improvements on harder multi-needle tasks (e.g., up to 350\% relative improvement on S-NIAH-3 and 67\% on MQ-NIAH). These results indicate that \sysname more effectively preserves and aggregates long-range information under extended contexts.

\begin{wraptable}{r}{0.60\textwidth}
\vspace{-2mm}
\centering
\caption{Performance on LongBench datasets~\citep{DBLP:conf/acl/BaiLZL0HDLZHDTL24} with different types of tasks.
}
\resizebox{0.60\textwidth}{!}{%
\begin{tabular}{r|ccc|ccc|ccc|ccc}
\toprule
& \multicolumn{3}{c|}{\textbf{Single-Doc QA}} & \multicolumn{3}{c|}{\textbf{Multi-Doc QA}} & \multicolumn{3}{c|}{\textbf{Summarization}} & \multicolumn{3}{c}{\textbf{Few-shot Learning}}\\ \midrule
 \textbf{Model} & {NQA} & {QQA} & {MFQ} 
          & {HQA} & {2WM} & {Mus} 
          & {GvR} & {QMS} & {MNs}
          & {TRC} & {TQA} & {SSM} \\\midrule
Transformer         & 8.4 & 9.6 & 19.6 & 11.0 & 20.9 & 6.4 & 12.8 & 9.7 & 9.5 & 21.0 & 43.2 & 14.4\\
\midrule
Mamba-2              & 5.1 & 10.6 & 11.9 & 10.2 & 14.5 & 4.7 & 6.3 & 5.5 & 3.3 & 2.3 & 20.9 & 6.8 \\
w/ Log-Linear & 6.8 & 9.6 & 12.2 & 9.5 & 19.0 & 4.4 & 6.4 & 8.1 & 3.5 & 12.4 & 18.6 & 9.8 \\
\rowcolor{black!10}
w/ \sysname          &
\textbf{9.4 (\textcolor{ForestGreen}{$\uparrow38\%$})} &
\textbf{11.1 (\textcolor{ForestGreen}{$\uparrow5\%$})} &
\textbf{16.1 (\textcolor{ForestGreen}{$\uparrow32\%$})} &
\textbf{12.5 (\textcolor{ForestGreen}{$\uparrow23\%$})} &
\textbf{23.9 (\textcolor{ForestGreen}{$\uparrow26\%$})} &
\textbf{8.7 (\textcolor{ForestGreen}{$\uparrow85\%$})} &
\textbf{8.3 (\textcolor{ForestGreen}{$\uparrow30\%$})} &
\textbf{12.2 (\textcolor{ForestGreen}{$\uparrow51\%$})} &
\textbf{9.3 (\textcolor{ForestGreen}{$\uparrow266\%$})} &
\textbf{18.7 (\textcolor{ForestGreen}{$\uparrow51\%$})} &
\textbf{31.4 (\textcolor{ForestGreen}{$\uparrow50\%$})} &
\textbf{16.5 (\textcolor{ForestGreen}{$\uparrow68\%$})}\\
\midrule
Gated DeltaNet       & 7.2 & 10.3 & 14.0 & 9.8 & 18.5 & 6.3 & 7.2 & 8.4 & 7.6 & 16.5 & 25.3 & 11.0\\
w/ Log-Linear & 6.9 & 6.1 & 15.4 & 11.5 & 20.2 & 5.5 & 6.1 & 9.9 & 4.3 & 11.0 & 27.7 & 11.2 \\
\rowcolor{black!10}
w/ \sysname          &
\textbf{11.3 (\textcolor{ForestGreen}{$\uparrow57\%$})} &
\textbf{15.1 (\textcolor{ForestGreen}{$\uparrow47\%$})} &
\textbf{17.4 (\textcolor{ForestGreen}{$\uparrow24\%$})} &
\textbf{17.2 (\textcolor{ForestGreen}{$\uparrow50\%$})} &
\textbf{25.4 (\textcolor{ForestGreen}{$\uparrow26\%$})} &
\textbf{7.0 (\textcolor{ForestGreen}{$\uparrow11\%$})} &
\textbf{8.7 (\textcolor{ForestGreen}{$\uparrow21\%$})} &
\textbf{11.8 (\textcolor{ForestGreen}{$\uparrow19\%$})} &
\textbf{9.9 (\textcolor{ForestGreen}{$\uparrow30\%$})} &
\textbf{31.2 (\textcolor{ForestGreen}{$\uparrow89\%$})} &
\textbf{41.5 (\textcolor{ForestGreen}{$\uparrow50\%$})} &
\textbf{19.7 (\textcolor{ForestGreen}{$\uparrow76\%$})}\\
\bottomrule
\end{tabular}
}
\label{tab:long_bench}
\vspace{-2mm}
\end{wraptable}

Second, on LongBench, \sysname consistently outperforms both vanilla and log-linear variants across diverse long-context understanding tasks, including narrative QA, multi-field QA, summarization, and few-shot learning. Notably, \sysname delivers strong and uniform gains across different task categories, suggesting that the benefits of \sysname extend beyond retrieval and generalize to complex reasoning and generation under long contexts.



\subsection{Inference Efficiency of DLA}
\begin{figure*}[t]
\centering
 \includegraphics[width=0.65\textwidth]{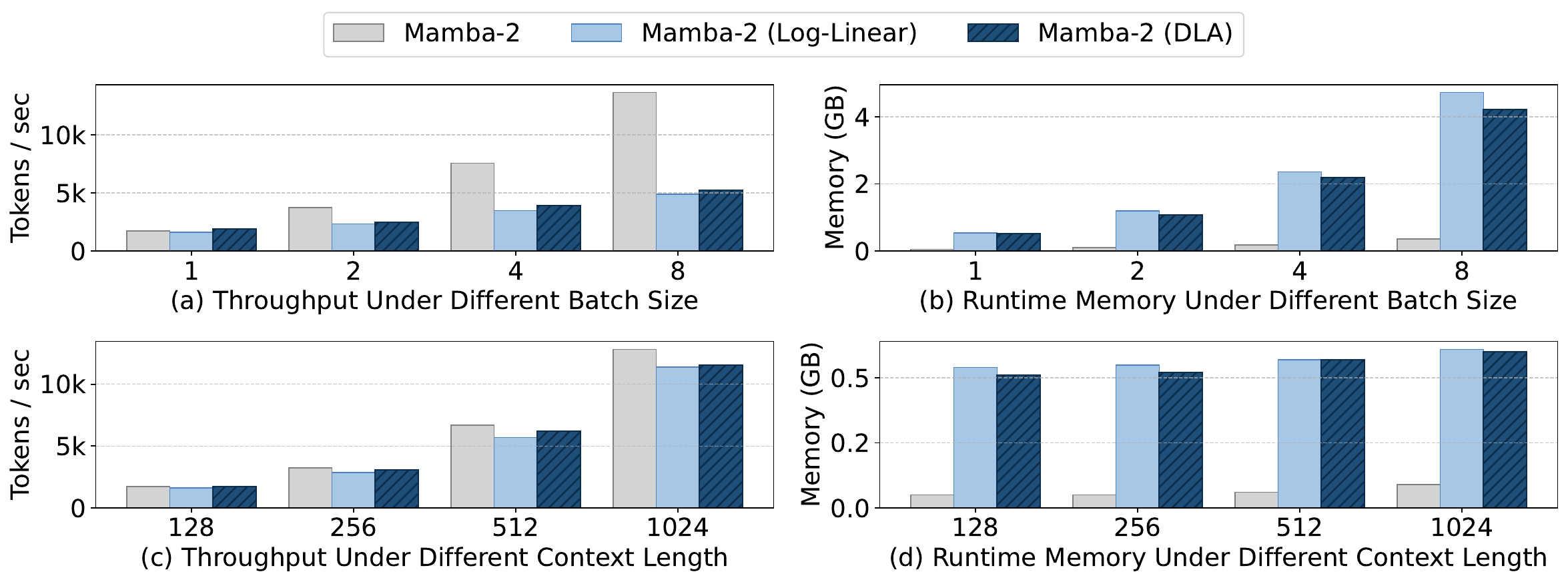}
\caption{Throughput (tokens/sec) and runtime memory consumption (GB) of vanilla, Log-Linear, and \sysname variants of Mamba-2 (780M) in prefill stage on a single A100 GPU under different bat ch sizes (a, b) and different sequence lengths (c, d).}
\label{fig:efficiency}
\end{figure*}

We next evaluate the efficiency of \sysname from two aspects: (1) efficiency under varying batch sizes and (2) efficiency under varying input context lengths.

\textbf{Efficiency Under Various Batch Sizes.} 
Figures~\cref{fig:efficiency}(a) and (c) report the throughput and runtime memory footprint under varying batch sizes with a fixed context length of 128 and decode length of 1. As batch size increases, both the log-linear and \sysname variants exhibit smaller throughput gains and higher memory usage than the vanilla Mamba-2, due to caching multiple summary states. Nevertheless, compared to log-linear attention, \sysname consistently achieves higher throughput with lower memory consumption, indicating better compute and memory efficiency.

\textbf{Efficiency Under Various Context Lengths.}
Figures~\cref{fig:efficiency}(b) and (d) show the throughput and KV memory footprint under varying context lengths with a fixed batch size of 1 and decode length of 1. As context length increases, both log-linear and \sysname variants incur higher memory usage and limited throughput improvement relative to the vanilla model, again due to maintaining multiple summary states. In contrast, \sysname consistently outperforms the log-linear variant in throughput while maintaining lower and more stable memory consumption, demonstrating superior efficiency under long-context settings.

\subsection{Ablation Studies}

\begin{wraptable}{r}{0.50\textwidth}
\vspace{-2mm}
\centering
\caption{Ablation study of Mamba-2 and Gated DeltaNet with different variants. \texttt{DLA(I)}\xspace denotes the version of \sysname with information-aware dynamic state merging only.}
\label{tab:ablation}
\resizebox{0.48\textwidth}{!}{%
\begin{tabular}{r|ccccc}
\toprule
\textbf{Model}
& \textbf{LMB.\textcolor{ForestGreen}{$\uparrow$}} 
& \textbf{PIQA\textcolor{ForestGreen}{$\uparrow$}} 
& \textbf{Hella.\textcolor{ForestGreen}{$\uparrow$}} 
& \textbf{Wino.\textcolor{ForestGreen}{$\uparrow$}} 
& \textbf{RULER\textcolor{ForestGreen}{$\uparrow$}}  \\
\midrule
Mamba-2   & 15.7 & 58.9 & 29.3 & 50.1 & 28.3 \\
 w/ Log-linear   & 13.2 & 59.7 & 27.8 & 49.5 & 39.0 \\
\rowcolor{gray!15}
  w/ \texttt{DLA(I)}   & 20.1 (\textcolor{ForestGreen}{$\uparrow28\%$}) & 61.5 (\textcolor{ForestGreen}{$\uparrow3\%$}) & 30.3 (\textcolor{ForestGreen}{$\uparrow3\%$}) & 50.9 (\textcolor{ForestGreen}{$\uparrow2\%$}) & 44.0 (\textcolor{ForestGreen}{$\uparrow13\%$}) \\
\rowcolor{gray!15}
  w/ \sysname   & 23.9 (\textcolor{ForestGreen}{$\uparrow52\%$}) & 63.7 (\textcolor{ForestGreen}{$\uparrow7\%$}) & 30.8 (\textcolor{ForestGreen}{$\uparrow5\%$}) & 51.5 (\textcolor{ForestGreen}{$\uparrow3\%$}) & 57.8 (\textcolor{ForestGreen}{$\uparrow48\%$}) \\
\midrule
Gated DeltaNet & 20.3 & 58.8 & 29.6 & 51.3 & 46.8 \\
 w/ Log-linear   & 19.0 & 60.4 & 28.4 & 51.9 & 55.4 \\
\rowcolor{gray!15}
  w/ \texttt{DLA(I)}   & 22.2 (\textcolor{ForestGreen}{$\uparrow9\%$}) & 61.3 (\textcolor{ForestGreen}{$\uparrow1\%$}) & 29.7 (\textcolor{ForestGreen}{$\uparrow5\%$}) & 52.0 (\textcolor{ForestGreen}{$\uparrow1\%$}) & 58.4 (\textcolor{ForestGreen}{$\uparrow5\%$}) \\
\rowcolor{gray!15}
 w/ \sysname   & 24.8 (\textcolor{ForestGreen}{$\uparrow22\%$}) & 62.7 (\textcolor{ForestGreen}{$\uparrow4\%$}) & 30.3 (\textcolor{ForestGreen}{$\uparrow6\%$}) & 52.4 (\textcolor{ForestGreen}{$\uparrow1\%$}) & 63.1 (\textcolor{ForestGreen}{$\uparrow14\%$}) \\
\bottomrule
\end{tabular}%
}
\vspace{-2mm}
\end{wraptable}

\textbf{Module Sensitivity Study.} We conduct ablation studies to evaluate the separate contribution of the two components of \sysname. Let \texttt{DLA(I)}\xspace denote the version of \sysname with information-aware dynamic state merging only. As shown in Table~\ref{tab:ablation}, we have two observations. (1) \texttt{DLA(I)}\xspace and \sysname variants consistently outperform Log-Linear variants across all benchmarks. (2) \sysname consistently outperforms \texttt{DLA(I)}\xspace across all benchmarks. This result demonstrates the unique contributions of the two components in \sysname.

\begin{wraptable}{r}{0.50\textwidth}
\vspace{-2mm}
\centering
\caption{Ablation study of Mamba-2 \sysname variant with different memory budget $k$ and merge boundary $\tau$.}
\label{tab:budget}
\resizebox{0.48\textwidth}{!}{%
\begin{tabular}{r|ccccc}
\toprule
\textbf{Budget Size}
& \textbf{LMB.\textcolor{ForestGreen}{$\uparrow$}} 
& \textbf{PIQA\textcolor{ForestGreen}{$\uparrow$}} 
& \textbf{Hella.\textcolor{ForestGreen}{$\uparrow$}} 
& \textbf{Wino.\textcolor{ForestGreen}{$\uparrow$}} 
& \textbf{RULER\textcolor{ForestGreen}{$\uparrow$}}  \\
\midrule
\texttt{$k$ = 20}   & 22.7 & 64.9 & 29.3 & 50.1 & 56.3 \\
\texttt{$k$ = 30}   & 23.9  & 63.7& 30.8  & 51.5  & 57.8\\
  \texttt{$k$ = 40}   & 24.2  & 63.3 & 29.8  & 50.9  & 57.7  \\
\midrule
\texttt{$\tau$ = 0.5}   & 20.1 & 63.6 & 29.7 & 51.3 & 57.3 \\
\texttt{$\tau$ = 0.6}   & 23.9  & 63.7& 30.8  & 51.5  & 57.8\\
  \texttt{$\tau$ = 0.7}   & 22.5  & 61.7 & 30.1  & 52.3  & 55.9  \\
\bottomrule
\end{tabular}%
}
\vspace{-2mm}
\end{wraptable}

\textbf{Impact of capacity $k$ and boundary $\tau$.} To study the impact of memory budget $k$ and the merge boundary $\tau$ on performance, we adjust the default budget in \sysname variant of Mamba-2 from 30 to 20 and 40 and adjust the default boundary from 0.6 to 0.5 and 0.7. We then compare the changes in performance. As shown in \cref{tab:budget}, changes in the memory budget and merge boundary have only a marginal effect on the final performance of \sysname, indicating that the proposed memory modeling is robust to these two hyperparameters.

\section{Related Work}
\label{sec:related_works}

To overcome the quadratic bottleneck of softmax attention on long sequences, linear attention and state space models (SSMs) reformulate attention computation to achieve $O(T)$ complexity.
Representative methods such as DeltaNet~\citep{DBLP:conf/nips/YangWZSK24} and Mamba~\citep{DBLP:journals/corr/abs-2312-00752} compress the entire history into a single recurrent state, continuously merging incoming tokens into a fixed-size summary for inference.
To alleviate the resulting over-compression, gating mechanisms~\citep{DBLP:conf/iclr/YangKH25} introduce data-dependent modulation to selectively attenuate obsolete information.
More recent approaches extend linear attention to multi-state memory.
In particular, Log-Linear Attention~\citep{DBLP:journals/corr/abs-2506-04761} maintains a logarithmic number of hierarchical states, where tokens are deterministically merged according to a fixed temporal schedule.
Despite their effectiveness, these methods rely on \emph{fixed merging policies} that ignore token-level information variation, leaving open the question of how to adaptively control state construction to preserve fine-grained information under long contexts.

\vspace{-2mm}
\section{Conclusion}
\label{sec:conclusion}

In this paper, we presented \sysname, a framework for multi-state linear attention. \sysname replaces fixed merging with information-aware dynamic state construction and uses capacity-bounded memory modeling to keep inference cost predictable. By allocating memory resolution based on token-level information variation, \sysname improves representation quality while preserving efficiency. We pre-train \sysname on two linear-attention backbones and evaluate it on 16 datasets across three aspects, where it consistently outperforms state-of-the-art baselines.

\vspace{-2mm}
\section*{Acknowledgement}
\vspace{-2mm}

This work is supported in part by NSF Award NeTS-2312675.

\bibliography{Reference}

@article{arora2023language,
  title={Language Models Enable Simple Systems for Generating Structured Views of Heterogeneous Data Lakes},
  author={Arora, Simran and Yang, Brandon and Eyuboglu, Sabri and Narayan, Avanika and Hojel, Andrew and Trummer, Immanuel and R{\'e}, Christopher},
  journal={Proceedings of the VLDB Endowment},
  volume={17},
  number={2},
  pages={92--105},
  year={2023},
  publisher={VLDB Endowment}
}

@inproceedings{lockard2019openceres,
  title={Openceres: When open information extraction meets the semi-structured web},
  author={Lockard, Colin and Shiralkar, Prashant and Dong, Xin Luna},
  booktitle={Proceedings of the 2019 Conference of the North American Chapter of the Association for Computational Linguistics: Human Language Technologies, Volume 1 (Long and Short Papers)},
  pages={3047--3056},
  year={2019}
}

@inproceedings{rajpurkar2018know,
  title={Know What You Don’t Know: Unanswerable Questions for SQuAD},
  author={Rajpurkar, Pranav and Jia, Robin and Liang, Percy},
  booktitle={Proceedings of the 56th Annual Meeting of the Association for Computational Linguistics (Volume 2: Short Papers)},
  pages={784--789},
  year={2018}
}

@article{kwiatkowski2019natural,
  title={Natural questions: a benchmark for question answering research},
  author={Kwiatkowski, Tom and Palomaki, Jennimaria and Redfield, Olivia and Collins, Michael and Parikh, Ankur and Alberti, Chris and Epstein, Danielle and Polosukhin, Illia and Devlin, Jacob and Lee, Kenton and others},
  journal={Transactions of the Association for Computational Linguistics},
  volume={7},
  pages={453--466},
  year={2019},
  publisher={MIT Press One Rogers Street, Cambridge, MA 02142-1209, USA journals-info~…}
}

@inproceedings{joshi2017triviaqa,
  title={TriviaQA: A Large Scale Distantly Supervised Challenge Dataset for Reading Comprehension},
  author={Joshi, Mandar and Choi, Eunsol and Weld, Daniel S and Zettlemoyer, Luke},
  booktitle={Proceedings of the 55th Annual Meeting of the Association for Computational Linguistics (Volume 1: Long Papers)},
  pages={1601--1611},
  year={2017}
}

@inproceedings{dua2019drop,
  title={DROP: A Reading Comprehension Benchmark Requiring Discrete Reasoning Over Paragraphs},
  author={Dua, Dheeru and Wang, Yizhong and Dasigi, Pradeep and Stanovsky, Gabriel and Singh, Sameer and Gardner, Matt},
  booktitle={Proceedings of the 2019 Conference of the North American Chapter of the Association for Computational Linguistics: Human Language Technologies, Volume 1 (Long and Short Papers)},
  pages={2368--2378},
  year={2019}
}

@inproceedings{arora2024simple,
  title={Simple linear attention language models balance the recall-throughput tradeoff},
  author={Arora, Simran and Eyuboglu, Sabri and Zhang, Michael and Timalsina, Aman and Alberti, Silas and Zou, James and Rudra, Atri and R{\'e}, Christopher},
  booktitle={Proceedings of the 41st International Conference on Machine Learning},
  pages={1763--1840},
  year={2024}
}

@inproceedings{DBLP:conf/nips/YangWZSK24,
  author       = {Songlin Yang and
                  Bailin Wang and
                  Yu Zhang and
                  Yikang Shen and
                  Yoon Kim},
  title        = {Parallelizing Linear Transformers with the Delta Rule over Sequence
                  Length},
  booktitle    = {NeurIPS},
  year         = {2024}
}

@inproceedings{DBLP:conf/emnlp/MihaylovCKS18,
  author       = {Todor Mihaylov and
                  Peter Clark and
                  Tushar Khot and
                  Ashish Sabharwal},
  title        = {Can a Suit of Armor Conduct Electricity? {A} New Dataset for Open
                  Book Question Answering},
  booktitle    = {{EMNLP}},
  pages        = {2381--2391},
  publisher    = {Association for Computational Linguistics},
  year         = {2018}
}

@inproceedings{DBLP:conf/iclr/YangKH25,
  author       = {Songlin Yang and
                  Jan Kautz and
                  Ali Hatamizadeh},
  title        = {Gated Delta Networks: Improving Mamba2 with Delta Rule},
  booktitle    = {{ICLR}},
  publisher    = {OpenReview.net},
  year         = {2025}
}

@article{DBLP:journals/corr/abs-2312-00752,
  author       = {Albert Gu and
                  Tri Dao},
  title        = {Mamba: Linear-Time Sequence Modeling with Selective State Spaces},
  journal      = {CoRR},
  volume       = {abs/2312.00752},
  year         = {2023}
}

@article{DBLP:journals/corr/abs-2506-04761,
  author       = {Han Guo and
                  Songlin Yang and
                  Tarushii Goel and
                  Eric P. Xing and
                  Tri Dao and
                  Yoon Kim},
  title        = {Log-Linear Attention},
  journal      = {CoRR},
  volume       = {abs/2506.04761},
  year         = {2025}
}

@article{DBLP:journals/corr/abs-2507-04416,
  author       = {Xiuying Wei and
                  Anunay Yadav and
                  Razvan Pascanu and
                  Caglar Gulcehre},
  title        = {{RAT:} Bridging {RNN} Efficiency and Attention Accuracy in Language
                  Modeling},
  journal      = {CoRR},
  volume       = {abs/2507.04416},
  year         = {2025}
}

@article{johnston2025generalizing,
  title={Generalizing the Cauchy-Schwarz inequality: Hadamard powers and tensor products},
  author={Johnston, Nathaniel and Plosker, Sarah and Torrance, Charles and Varona, Luis},
  journal={arXiv preprint arXiv:2507.10327},
  year={2025}
}

@inproceedings{10.1145/1644893.1644914,
    author = {Lumezanu, Cristian and Baden, Randy and Spring, Neil and Bhattacharjee, Bobby},
    title = {Triangle inequality variations in the internet},
    year = {2009},
    isbn = {9781605587714},
    publisher = {Association for Computing Machinery},
    address = {New York, NY, USA},
    url = {https://doi.org/10.1145/1644893.1644914},
    doi = {10.1145/1644893.1644914},
    booktitle = {Proceedings of the 9th ACM SIGCOMM Conference on Internet Measurement},
    pages = {177–183},
    numpages = {7},
    location = {Chicago, Illinois, USA},
    series = {IMC '09}
}

@inproceedings{DBLP:conf/naacl/TalmorHLB19,
  author       = {Alon Talmor and
                  Jonathan Herzig and
                  Nicholas Lourie and
                  Jonathan Berant},
  title        = {CommonsenseQA: {A} Question Answering Challenge Targeting Commonsense
                  Knowledge},
  booktitle    = {{NAACL-HLT} {(1)}},
  pages        = {4149--4158},
  publisher    = {Association for Computational Linguistics},
  year         = {2019}
}

@misc{eval-harness,
  author       = {Gao, Leo and Tow, Jonathan and Abbasi, Baber and Biderman, Stella and Black, Sid and DiPofi, Anthony and Foster, Charles and Golding, Laurence and Hsu, Jeffrey and Le Noac'h, Alain and Li, Haonan and McDonell, Kyle and Muennighoff, Niklas and Ociepa, Chris and Phang, Jason and Reynolds, Laria and Schoelkopf, Hailey and Skowron, Aviya and Sutawika, Lintang and Tang, Eric and Thite, Anish and Wang, Ben and Wang, Kevin and Zou, Andy},
  title        = {The Language Model Evaluation Harness},
  month        = 07,
  year         = 2024,
  publisher    = {Zenodo},
  version      = {v0.4.3},
  doi          = {10.5281/zenodo.12608602},
  url          = {https://zenodo.org/records/12608602}
}

@inproceedings{DBLP:conf/acl/BaiLZL0HDLZHDTL24,
  author       = {Yushi Bai and
                  Xin Lv and
                  Jiajie Zhang and
                  Hongchang Lyu and
                  Jiankai Tang and
                  Zhidian Huang and
                  Zhengxiao Du and
                  Xiao Liu and
                  Aohan Zeng and
                  Lei Hou and
                  Yuxiao Dong and
                  Jie Tang and
                  Juanzi Li},
  title        = {LongBench: {A} Bilingual, Multitask Benchmark for Long Context Understanding},
  booktitle    = {{ACL} {(1)}},
  pages        = {3119--3137},
  publisher    = {Association for Computational Linguistics},
  year         = {2024}
}

@article{DBLP:journals/corr/abs-2404-06654,
  author       = {Cheng{-}Ping Hsieh and
                  Simeng Sun and
                  Samuel Kriman and
                  Shantanu Acharya and
                  Dima Rekesh and
                  Fei Jia and
                  Yang Zhang and
                  Boris Ginsburg},
  title        = {{RULER:} What's the Real Context Size of Your Long-Context Language
                  Models?},
  journal      = {CoRR},
  volume       = {abs/2404.06654},
  year         = {2024}
}

@article{DBLP:journals/corr/abs-2102-03315,
  author       = {Sumithra Bhakthavatsalam and
                  Daniel Khashabi and
                  Tushar Khot and
                  Bhavana Dalvi Mishra and
                  Kyle Richardson and
                  Ashish Sabharwal and
                  Carissa Schoenick and
                  Oyvind Tafjord and
                  Peter Clark},
  title        = {Think you have Solved Direct-Answer Question Answering? Try ARC-DA,
                  the Direct-Answer {AI2} Reasoning Challenge},
  journal      = {CoRR},
  volume       = {abs/2102.03315},
  year         = {2021}
}

@article{DBLP:journals/cacm/SakaguchiBBC21,
  author       = {Keisuke Sakaguchi and
                  Ronan Le Bras and
                  Chandra Bhagavatula and
                  Yejin Choi},
  title        = {WinoGrande: an adversarial winograd schema challenge at scale},
  journal      = {Commun. {ACM}},
  volume       = {64},
  number       = {9},
  pages        = {99--106},
  year         = {2021}
}

@inproceedings{DBLP:conf/acl/ZellersHBFC19,
  author       = {Rowan Zellers and
                  Ari Holtzman and
                  Yonatan Bisk and
                  Ali Farhadi and
                  Yejin Choi},
  title        = {HellaSwag: Can a Machine Really Finish Your Sentence?},
  booktitle    = {{ACL} {(1)}},
  pages        = {4791--4800},
  publisher    = {Association for Computational Linguistics},
  year         = {2019}
}

@inproceedings{DBLP:conf/aaai/BiskZLGC20,
  author       = {Yonatan Bisk and
                  Rowan Zellers and
                  Ronan Le Bras and
                  Jianfeng Gao and
                  Yejin Choi},
  title        = {{PIQA:} Reasoning about Physical Commonsense in Natural Language},
  booktitle    = {{AAAI}},
  pages        = {7432--7439},
  publisher    = {{AAAI} Press},
  year         = {2020}
}

@inproceedings{DBLP:conf/acl/PapernoKLPBPBBF16,
  author       = {Denis Paperno and
                  Germ{\'{a}}n Kruszewski and
                  Angeliki Lazaridou and
                  Quan Ngoc Pham and
                  Raffaella Bernardi and
                  Sandro Pezzelle and
                  Marco Baroni and
                  Gemma Boleda and
                  Raquel Fern{\'{a}}ndez},
  title        = {The {LAMBADA} dataset: Word prediction requiring a broad discourse
                  context},
  booktitle    = {{ACL} {(1)}},
  publisher    = {The Association for Computer Linguistics},
  year         = {2016}
}

@inproceedings{DBLP:conf/icml/DaoG24,
  author       = {Tri Dao and
                  Albert Gu},
  title        = {Transformers are SSMs: Generalized Models and Efficient Algorithms
                  Through Structured State Space Duality},
  booktitle    = {{ICML}},
  publisher    = {OpenReview.net},
  year         = {2024}
}

@inproceedings{DBLP:conf/iclr/WanWZXTZWLXW025,
  author       = {Zhongwei Wan and
                  Xinjian Wu and
                  Yu Zhang and
                  Yi Xin and
                  Chaofan Tao and
                  Zhihong Zhu and
                  Xin Wang and
                  Siqi Luo and
                  Jing Xiong and
                  Longyue Wang and
                  Mi Zhang},
  title        = {{D2O:} Dynamic Discriminative Operations for Efficient Long-Context
                  Inference of Large Language Models},
  booktitle    = {{ICLR}},
  publisher    = {OpenReview.net},
  year         = {2025}
}

@article{wan2023efficient,
    title={Efficient Large Language Models: A Survey},
    author={Wan, Zhongwei and Wang, Xin and others},
    year={2023},
    journal={arXiv preprint arXiv:2312.03863},
}

@article{wang2024iot,
    title={IoT in the Era of Generative AI: Vision and Challenges},
    author={Wang, Xin and Wan, Zhongwei and Hekmati, Arvin and Zong, Mingyu and Alam, Samiul and Zhang, Mi and Krishnamachari, Bhaskar},
    journal={arXiv preprint arXiv:2401.01923},
    year={2024}
}
\bibliographystyle{unsrtnat}

\end{document}